\documentclass[letterpaper, 10 pt, conference]{ieeeconf}

\IEEEoverridecommandlockouts
\usepackage{amsmath}
\usepackage{amssymb}

\usepackage{amsthm}

\usepackage{algorithm, algpseudocode}
\usepackage{cite}
\usepackage{nicefrac}
\usepackage{pgf, graphicx, subcaption}

\let\labelindent\relax
\usepackage{enumitem}

\DeclareMathOperator*{\argmin}{arg\,min}
\DeclareMathOperator*{\Argmin}{Arg\,min}

\makeatletter
\let\NAT@parse\undefined
\makeatother
\usepackage{hyperref}

\newtheorem{remark}{Remark}
\newtheorem{assumption}{Assumption}
\newtheorem{lemma}{Lemma}

\newtheorem{theorem}{Theorem}

\title{\LARGE \bf
Decentralized Event-Triggered Federated Learning \\ with Heterogeneous Communication Thresholds
}

\author{Shahryar Zehtabi, Seyyedali Hosseinalipour, Christopher G. Brinton
\thanks{S. Zehtabi, S. Hosseinalipour and C. Brinton are with the School of Electrical and Computer Engineering, Purdue University, West Lafayette, IN, 47906 USA
        email:{\tt\small $\lbrace$szehtabi,hosseina,cgb$\rbrace$@purdue.edu}}
\thanks{This work was supported in part by the Office of Naval Research under grant N00014-21-1-2472}
}

\begin{document}

\maketitle
\thispagestyle{empty}
\pagestyle{empty}

\begin{abstract}
A recent emphasis of distributed learning research has been on federated learning (FL), in which model training is conducted by the data-collecting devices.
Existing research on FL has mostly focused on a star topology learning architecture with synchronized (time-triggered) model training rounds, where the local models of the devices are periodically aggregated by a centralized coordinating node.
However, in many settings, such a coordinating node may not exist, motivating efforts to fully decentralize FL.
In this work, we propose a novel methodology for distributed model aggregations via asynchronous, event-triggered consensus iterations over the network graph topology.
We consider heterogeneous communication event thresholds at each device that weigh the change in local model parameters against the available local resources in deciding the benefit of aggregations at each iteration.
Through theoretical analysis, we demonstrate that our methodology achieves asymptotic convergence to the globally optimal learning model under standard assumptions in distributed learning and graph consensus literature, and without restrictive connectivity requirements on the underlying topology. Subsequent numerical results demonstrate that our methodology obtains substantial improvements in communication requirements compared with FL baselines.
\end{abstract}

\section{Introduction}
Federated learning (FL) has emerged as a popular technique for distributing machine learning model training across network devices \cite{kairouz2021advances}. In the conventional FL architecture, a set of devices are connected to a central coordinating node (e.g., an edge server) in a star topology configuration. Devices conduct local model updates based on their individual datasets, and the coordinator periodically aggregates these into a global model, synchronizing the devices to begin the next round of training. Several works in the past few years have built functionality into this architecture to manage different types of network heterogeneity, including varying communication and computation abilities of devices and statistical properties of local datasets \cite{bonawitz2019towards,li2020federated}.

However, access to a central coordinating node is not always feasible/desirable. For instance, ad-hoc wireless networks serve as an efficient alternative for communication among devices in settings where device-to-server connectivity is energy intensive or unavailable \cite{chiang2016fog}. The proliferation of such settings motivates consideration of \textit{fully-decentralized FL}, where the model aggregation step, in addition to the data processing step, is distributed across devices. In this paper, we propose a cooperative learning approach for achieving this via consensus iterations over the available distributed graph topology, and analyze its convergence characteristics.

The central coordinator in FL is also typically employed for timing synchronization, i.e., determining the time between global aggregations. To overcome this, we consider an \textit{asynchronous, event-triggered communication framework} for distributed model consensus. Event-triggered communications can offer several benefits in this context. For one, the amount of redundant communications can be reduced by defining event triggering conditions based on the significance of each device's model update. Also, removing the assumption of devices communicating at every iteration opens the possibility of alleviating straggler issues \cite{hosseinalipour2020federated}. Third, we can improve computational efficiency at each device by limiting aggregations to only when new parameters are received.

\subsection{Related Work}

\subsubsection{Consensus-based distributed optimization}
There is a rich literature on distributed optimization over graphs via consensus algorithms, e.g., \cite{tsitsiklis1986distributed, nedic2009distributed, nedic2010constrained, pu2021distributed, nedic2014distributed, xin2018linear}. For connected, undirected graph topologies, symmetric and doubly-stochastic transition matrices can be constructed for consensus iterations. In typical approaches \cite{tsitsiklis1986distributed, nedic2009distributed, nedic2010constrained}, each device maintains a local gradient of the target system objective (e.g., error minimization), with the consensus matrices designed to satisfy additional convergence criteria outlined in \cite{boyd2004fastest, xiao2004fast}. More recently, gradient tracking optimization techniques have been developed where the global gradient is simultaneously learned alongside local parameters \cite{pu2021distributed}.

In this work, our focus is on decentralized FL, which adds two unique aspects to the distributed optimization problem. First, the local data distributions across devices for machine learning tasks are in general not independent and identically distributed (non-i.i.d.), which can have significant impacts on convergence \cite{hosseinalipour2020federated}. Second, we consider the realistic scenario in which the devices have heterogeneous resources \cite{bonawitz2019towards}.

\subsubsection{Resource-efficient federated learning}
Several recent works in FL have investigated techniques for improving communication and computation efficiency. A popular line of research has aimed to adaptively control the FL process based on device capabilities, e.g., \cite{nishio2019client,nguyen2020fast,diao2020heterofl,wang2019adaptive,gu2021fast}. \cite{wang2019adaptive} studies FL convergence under a total network resource budget, in which the server adapts the frequency of global aggregation iterations. Others \cite{nishio2019client,nguyen2020fast,gu2021fast} have considered FL under partial device participation, where the communication and processing capabilities of devices are taken into account when assessing which clients will participate in each training round. \cite{diao2020heterofl} remove the necessity that every local client needs to share the same global model as the server, allowing weaker clients to take smaller subsets of the model to optimize.

Different from these works, we focus on novel learning topologies for decentralized FL. In this respect, a few recent works \cite{savazzi2020federated,lalitha2019peer,hosseinalipour2022multi,lin2021semi} have proposed peer-to-peer (P2P) communication approaches for collaborative learning over local device topologies. \cite{hosseinalipour2022multi,lin2021semi} investigated a semi-decentralized FL methodology across hierarchical networks, where local model aggregations are conducted via P2P-based cooperative consensus formation to reduce the frequency of global aggregations by the coordinating node. In our work, we consider the fully decentralized setting, where a central node is not available, as in \cite{savazzi2020federated,lalitha2019peer}: alongside local model updates, devices conduct consensus iterations with their neighbors in order to gradually minimize the global machine learning loss in a distributed manner. Different from \cite{savazzi2020federated,lalitha2019peer}, our methodology incorporates asynchronous event-triggered communications among devices, where local resource levels are factored in to the event thresholds to account for device heterogeneity. We will see that this approach leads to substantial improvements in model convergence time compared with non-heterogeneous thresholding.

\vspace{-1mm}
\subsection{Outline and Summary of Contributions} \label{sec:intro:sub:contributions}

\begin{itemize}
\item We develop a novel methodology for fully decentralizing FL, with model aggregations occurring via cooperative model consensus iterations (Sec.~\ref{sec:method}). In our methodology, communications are asynchronous and event-driven. With event thresholds defined to incorporate local model evolution and resource availability, our methodology adapts to device communication and processing limitations in heterogeneous networks.

\item We provide a convergence analysis of our methodology, which shows that each device arrives at the globally optimal learning model asymptotically under standard assumptions for distributed learning (Sec.~\ref{sec:conv}). This result is obtained without overly restrictive connectivity assumptions on the underlying communication graph. Our analysis also leads to guardrails for the event-triggering conditions to ensure convergence.

\item We conduct numerical experiments comparing our methodology to baselines in decentralized FL and a randomized gossip algorithm on a real-world machine learning dataset (Sec.~\ref{sec:simulation}). We show that our method is able to reduce model training communication time substantially compared to FL baselines. Additionally, we find that the convergence rate of our method scales well with consensus graph connectivity.
\end{itemize}

\vspace{-1mm}
\section{Methodology and Algorithm}
\label{sec:method}
In this section, we develop our methodology for decentralized FL with event-triggered communications. After discussing preliminaries of the learning model in FL (Sec.~\ref{ssec:FL}), we present our cooperative consensus algorithm for distributed model aggregations (Sec.~\ref{ssec:event}). Finally, we remark on hyperparameters introduced in our algorithm (Sec.~\ref{ssec:remark}).

\subsection{Device and Learning Model}
\label{ssec:FL}
We consider a system of $m$ devices/nodes, collected via the set $\mathcal{M}$, which are engaged in distributed training of a machine learning model. Under the FL framework, each device $i\in\mathcal{M}$ trains a local model $\mathbf{w}_i$ using its own generated dataset $\mathcal{D}_i$. Each data point $\xi\triangleq(\mathbf{x}_\xi, y_\xi)\in\mathcal{D}_i$ consists of a feature vector $\mathbf{x}_\xi$ and (optionally, in the case of supervised learning) a target label $y_\xi$. The performance of the local model is measured via the
local loss $F_i{\left( . \right)}$:
\vspace{-1mm}
 \begin{equation}
     F_i{\left( \mathbf{w} \right)} = \sum_{\xi\in\mathcal{D}_i}\ell_\xi\left(\mathbf{w} \right),
     \vspace{-1mm}
 \end{equation}
where $\ell_\xi\left(\mathbf{w}_i \right)$ is the loss of the model on datapoint $\xi$ (e.g., squared prediction error) under parameter realization $\mathbf{w}\in \mathbb{R}^n$, with $n$ denoting the dimension of the target model.
The global loss is defined in terms of these local losses as
\vspace{-1mm}
 \begin{equation}
     F{\left(\mathbf{w} \right)} = \frac{1}{\left|\mathcal{M}  \right|}\sum_{i\in\mathcal{M}}{F_i{\left( \mathbf{w} \right)}}.
     \vspace{-1mm}
 \end{equation}
 
The goal of the training process is to find an optimal parameter vector $\mathbf{w}^\star$ which minimizes the global loss function, i.e., $\mathbf{w}^\star=\argmin_{\mathbf{w}\in\mathbb{R}^n} F{\left( \mathbf{w} \right)}$.
In the distributed setting, we desire $\mathbf{w}_1 = \cdots = \mathbf{w}_m = \mathbf{w}^\star$, which requires a synchronization mechanism. In conventional FL, synchronization is conducted periodically by a central coordinator globally aggregating the local models. However, in this work, we are interested in settings where no such central node exists. Thus, alongside using optimization techniques to minimize local loss functions, we must desire a technique to reach consensus over the parameters in our decentralized FL scheme.

To accomplish this, we propose event-triggered FL with heterogeneous communication thresholds ({\tt EF-HC}). In {\tt EF-HC}, devices conduct peer-to-peer (P2P) communications during the model training period to synchronize their locally trained models and avoid overfitting to their local datasets. The overall {\tt EF-HC} algorithm is given in Alg.~\ref{alg:eventBased}. Two model parameter vectors are kept at each device $i$: (i) its instantaneous \textit{main} model parameters $\mathbf{w}_i$, and (ii) the \textit{auxiliary} model parameters $\mathbf{\widehat{w}}_i$, which is the outdated version of its main parameters that had been broadcast to the neighbors. Decentralized ML is conducted over the (time-varying, undirected) device graph through a sequence of four events detailed in Sec.~\ref{ssec:event}. Although in our distributed setup there is no physical notion of a global iteration, we introduce the iteration variable $k$ for analysis purposes.

\subsection{Model Updating and Event-Triggered Communications}
\label{ssec:event}
We consider the physical network graph $\mathcal{G}^{(k)} = \left( \mathcal{M}, \mathcal{E}^{(k)} \right)$ among devices, where $\mathcal{E}^{(k)}$ is the set of edges available at iteration $k$ in the underlying time-varying communication graph. We assume that link availability varies over time according to the underlying communication protocol in place \cite{hosseinalipour2020federated}. In each iteration, some of these edges are employed for transmission/reception of model parameters between devices. To represent this process, we define the information flow graph $\mathcal{G}'^{(k)} = \left( \mathcal{M}, \mathcal{E}'^{(k)} \right)$, which is a subgraph of $\mathcal{G}^{(k)}$. $\mathcal{E}'^{(k)}$ only contains those links in $\mathcal{E}^{(k)}$ that are being used at iteration $k$ to exchange parameters. Based on this, we denote the neighbors of device $i$ at iteration $k$ as $\mathcal{N}_i^{(k)} = \lbrace  j :  (i,j)\in\mathcal{E}^{(k)}\, , j \in \mathcal{M}\}$, with node degree $d_i^{(k)} = |\mathcal{N}_i^{(k)}|$. We also denote neighbors of $i$ which are directly communicating with it at iteration $k$ as $\mathcal{N}'^{(k)}_i = \lbrace  j :  (i,j)\in\mathcal{E}'^{(k)}\, , j \in \mathcal{M}\}$. Additionally, the aggregation weight associated with the link $(i, j) \in \mathcal{E}^{(k)}$ and $(i, j) \in \mathcal{E}'^{(k)}$ are defined as $\beta_{ij}^{(k)}$ and $p_{ij}^{(k)}$ respectively, with $p_{ij}^{(k)} = \beta_{ij}^{(k)}$ if the link $(i, j)$ is used for aggregation at iteration $k$, and $p_{ij}^{(k)} = 0$ otherwise.

In {\tt EF-HC}, there are four types of communication events:

\textbf{\textit{Event 1: Neighbor connection.} } The first event (lines \ref{event:conn:begin}-\ref{event:conn:end} of Alg. \ref{alg:eventBased}) is triggered at device $i$ if new devices connect to it or existing devices disconnect from it due to the time-varying nature of the graph. In this event, model parameters $\mathbf{w}_i^{(k)}$ and the degree of device $i$ at that time $d_i^{(k)}$ are exchanged with this new neighbor. Consequently, this results in an aggregation event (Event 3) at both devices.

\textbf{\textit{Event 2: Broadcast.} } Second, if the normalized difference between $\mathbf{w}_i^{(k)}$ and $\mathbf{\widehat{w}}_i^{(k)}$ at device $i$ is greater than a \textit{threshold} value $r \rho_i \gamma^{(k)}$, i.e., ${( \frac1n )}^\frac1q {\| \mathbf{w}_i^{(k)} - \mathbf{\widehat{w}}_i^{(k)} \|}_q \geq r \rho_i \gamma^{(k)}$, then a broadcast event (lines \ref{event:broadcast:begin}-\ref{event:broadcast:end} of Alg. \ref{alg:eventBased}) is triggered at that device. In other words, communication at a device is triggered once the instantaneous local model is sufficiently different than the outdated local model. When this event triggers, device $i$ broadcasts its parameters $\mathbf{w}_i^{(k)}$ and its instantaneous degree $d_i^{(k)}$ to all of its neighbors and receives the same information from them.

The threshold $r \rho_i \gamma^{(k)}$ is treated as heterogeneous across devices $i\in\mathcal{M}$, to assess whether the gain from a consensus iteration on the instantaneous main models at the devices will be worth the induced network resource utilization. Specifically: (i) $r > 0$ is a scaling hyperparameter value; (ii) $\gamma^{(k)} > 0$ is a decaying factor that accounts for smaller expected variations in the local models over time; and (iii) $\rho_i$ quantifies the resource availability of device $i$.
Developing the threshold measure ${( \frac1n )}^\frac1q {\| \mathbf{w}_i^{(k)} - \mathbf{\widehat{w}}_i^{(k)} \|}_q$ and the condition $r \rho_i \gamma^{(k)}$ is one of the contributions of this paper relative to existing event-triggered schemes~\cite{george2020distributed}. For example, in a bandwidth limited environment, the transmission delay of model transfer will be inversely proportional to the bandwidth among two devices. Thus, to decrease the latency of model training,
$\rho_i$ can be defined inversely proportional to the bandwidth, promoting lower frequency of communications at the devices with less available bandwidth. In {\tt EF-HC}, we set $\rho_i \propto \frac1{b_i}$, where $b_i$ is the average bandwidth on outgoing links of device $i$. Further details on choosing the broadcast threshold are given in Sec.~\ref{ssec:remark}.

\textbf{\textit{Event 3: Aggregation.} } Following a broadcast event (Event 2) or a neighbor connection event (Event 1) at device $i$, an aggregation event (lines \ref{event:agg:begin}-\ref{event:agg:end} of Alg. \ref{alg:eventBased}) is triggered at device $i$ and all of its neighbors. This aggregation is carried out through a distributed weighted averaging consensus method~\cite{xiao2004fast} as $\mathbf{w}_i^{(k+1)} = \mathbf{w}_i^{(k)} + \sum_{j \in \mathcal{N}'^{(k)}_i}{\beta_{ij}^{(k)} \left( \mathbf{w}_j^{(k)} - \mathbf{w}_i^{(k)} \right)}$,
where $\beta_{ij}^{(k)}$ is the aggregation weight that device $i$ will assign to parameters received from device $j$ at iteration $k$. The aggregation weights $\{\beta_{ij}^{(k)}\}$ for graph $\mathcal{G}^{(k)}$ can be selected based on degree information of the neighbors, as will be discussed in Sec.~\ref{sec:convergence:sub:assumptions}.

\textbf{\textit{Event 4: Gradient descent.} } Each device $i$ conducts stochastic gradient descent (SGD) iterations for local model training. Formally, device $i$ obtains $\mathbf{w}_i^{(k+1)}= \mathbf{w}_i^{(k)} - \alpha^{(k)} \mathbf{g}_i{( \mathbf{w}_i^{(k)} )}$, where
$\alpha^{(k)}$ is the learning rate and $\mathbf{g}_i{ \mathbf{w}_i^{(k)} )}$ is the stochastic gradient approximation defined as $\mathbf{g}_i{( \mathbf{w}_i^{(k)} )}= \frac1{| \mathcal{S}_i^{(k)} |} \sum_{\mathbf{\xi} \in \mathcal{S}_i^{(k)}} \nabla \ell_\xi{( \mathbf{w}_i^{(k)} )}$. Here, $\mathcal{S}_i^{(k)}$ denotes the set of data points (mini-batch) used to compute the gradient, chosen uniformly at random from the local dataset.

\begin{algorithm}[t]
    \small
	\caption{{\tt EF-HC} procedure for device $i$.}
	\label{alg:eventBased}
	\textbf{Input:} $K, q$\\
	Initialize $k=0$, $\mathbf{w}_i^{(0)} = \mathbf{\widehat{w}}_i^{(0)}$
	\begin{algorithmic}[1]
		\While{$k\leq K$}
		
			\Comment{\textbf{Event 1. } Neighbor Connection Event}
			\If {device $j$ is connected to device $i$} \label{event:conn:begin}
				\State device $i$ appends device $j$ to its list of neighbors
				\State device $i$ sends $\mathbf{w}_i^{(k)}$ and $d_i^{(k)}$ to device $j$
				\State device $i$ receives $\mathbf{w}_j^{(k)}$ and $d_j^{(k)}$ from device $j$
			\ElsIf {device $j$ is disconnected from device $i$}
				\State device $i$ removes device $j$ from its list of neighbors
			\EndIf \label{event:conn:end}
		
			\Comment{\textbf{Event 2. } Broadcast Event}
			\If {${\left( \frac1n \right)}^\frac1q {\left\| \mathbf{w}_i^{(k)} - \mathbf{\widehat{w}}_i^{(k)} \right\|}_q \geq r \rho_i \gamma^{(k)}$} \label{event:broadcast:begin}
				\State device $i$ broadcasts $\mathbf{w}_i^{(k)}$, $d_i^{(k)}$ to all neighbors $j\in\mathcal{N}_i^{(k)}$
				\State device $i$ receives $\mathbf{w}_j^{(k)}$, $d_j^{(k)}$ from all neighbors $j\in\mathcal{N}_i^{(k)}$
				\State $\mathbf{\widehat{w}}_i^{(k+1)} = \mathbf{w}_i^{(k)}$
			\EndIf \label{event:broadcast:end}
			
			\Comment{\textbf{Event 3. } Aggregation Event}
			\If {updated parameters $\mathbf{w}_j^{(k)}$ and $d_j^{(k)}$ received from neighbor $j$} \label{event:agg:begin}
				\State $\mathbf{w}_i^{(k+1)} = \mathbf{w}_i^{(k)} + \sum_{j \in \mathcal{N}'^{(k)}_i}{\beta_{ij}^{(k)} \left( \mathbf{w}_j^{(k)} - \mathbf{w}_i^{(k)} \right)} $
			\EndIf \label{event:agg:end}
			
			\Comment{\textbf{Event 4. } Gradient Descent Event}
			\State device $i$ conducts SGD iteration $\mathbf{w}_i^{(k+1)} = \mathbf{w}_i^{(k)} - \alpha^{(k)} \mathbf{g}_i{\left( \mathbf{w}_i^{(k)} \right)}$

			\State $k \leftarrow k+1$ \label{event:sgd:end}
		\EndWhile
	\end{algorithmic}
\end{algorithm}

\vspace{-1mm}
\section{Convergence Analysis}
\label{sec:conv}
In this section, we first present our main theoretical result in this paper (Sec.~\ref{subsec:mainRes}). We then enumerate and discuss the assumptions needed to obtain the main result (Sec.~\ref{sec:convergence:sub:assumptions}).

\vspace{-1mm}
\subsection{Main Convergence Result}\label{subsec:mainRes}
We first obtain the convergence characteristics of {\tt EF-HC}. We reveal that (a) all devices reach consensus asymptotically, i.e., each device $i$'s model $\mathbf{w}_i^{(k)}$ converges to $\mathbf{\bar{w}}^{(k)}=\frac1m \sum_{i=1}^m{\mathbf{w}_i^{(k)}}$ as $k \to \infty$, and (b) the final model across the devices (i.e., $\mathbf{\bar{w}}^{(k)}, k\rightarrow \infty$) minimizes the global loss.
\begin{theorem} \label{theorem:main}
 Under the standard distributed learning assumptions in Sec.~\ref{sec:convergence:sub:assumptions}, model training under {\tt EF-HC} satisfies the following convergence behaviors:
 \vspace{-1mm}
    \begin{enumerate}[label=(\alph*)]
        \item
        $
            \lim_{k \to \infty}{{\| \mathbf{w}_i^{(k)} - \mathbf{\bar{w}}^{(k)} \|}_2} = 0$
        for all $i$, where $\mathbf{\bar{w}}^{(k)} = \frac1m \sum_{i=1}^m{\mathbf{w}_i^{(k)}}$, and \label{theorem:main:consensus}
        
        \item $\lim_{k \to \infty}{F{\left( \mathbf{\bar{w}}^{(k)} \right)} - F^\star} = 0$. \label{theorem:main:optimization}
    \end{enumerate}
\end{theorem}
\begin{proof}
See Appendix~\ref{ssec:proof}.
\end{proof}

\vspace{-1mm}
\subsection{Assumptions for Theorem 1} \label{sec:convergence:sub:assumptions}

\begin{assumption}[Simultaneous information exchange] \label{assump:simultaneousInfo}
The devices exchange information simultaneously: if device $j$ communicates with device $i$ at some time, device $i$ also communicates with device $j$ at that same time.
\end{assumption}

\begin{assumption}[Transition weights] \label{assump:weights}
Let $\{p_{ij}^{(k)}\}$ be the set of aggregation weights in the information graph $\mathcal{G}'(k)$. $p_{ij}^{(k)}$ is the transition weight that device $i$ utilizes to aggregate device $j$'s parameters at iteration $k$:
\vspace{-1mm}
\begin{equation} \label{eqn:pDefinition}
	p_{ij}^{(k)} =
	\begin{cases}
		\beta_{ij}^{(k)} v_{ij}^{(k)} & \quad i \neq j
		\\
		1 - \sum_{j=1}^m{\beta_{ij}^{(k)} v_{ij}^{(k)}} & \quad i = j
	\end{cases},
	\vspace{-1mm}
\end{equation}
where $v_i^{(k)}$ indicates whether a broadcast event has occurred at device $i$ at iteration $k$:
\vspace{-1mm}
\begin{equation} \label{eqn:indicatorSignal}
	\begin{gathered}
		v_i^{(k)} =
		\begin{cases}
			1 & {\left( \frac1n \right)}^\frac1q {\left\| \mathbf{e}_i^{(k)} \right\|}_q > r \rho_i \gamma^{(k)}
			\\
			0 & \text{o.w.}
		\end{cases}, 
		\\
		\mathbf{e}_i^{(k)} = \mathbf{w}_i^{(k)} - \mathbf{\widehat{w}}_i^{(k)}, \quad \rho_i = \frac1{b_i},
		\\
		v_{ij}^{(k)} =
		\begin{cases}
			\max{\lbrace v_i^{(k)}, v_j^{(k)} \rbrace} & j \in \mathcal{N}_i^{(k)}
			\\
			0 & \text{o.w.}
		\end{cases}.
	\end{gathered}
\end{equation}
The following conditions must hold:
	\begin{enumerate}[label=(\alph*)]
		\item (Non-negative weights) There exists a scalar $\eta$, $0 < \eta < 1$, such that $\forall i \in \mathcal{M}$, we have
		\begin{enumerate}[label=(\roman*)]
			\item $p_{ii}^{(k)} \geq \eta$ and $p_{ij}^{(k)} \geq \eta$ for all $k \geq 0$ and all neighbor devices $j\in\mathcal{N}'^{(k)}_i$.
			\item $p_{ij}^{(k)} = 0$, if $j\notin\mathcal{N}'^{(k)}_i$.
		\end{enumerate} \label{assump:weights:eta}
		
		\item (Doubly-stochastic weights) The rows and columns of matrix $\mathbf{P}^{(k)} = [p_{ij}^{(k)}]$ are both stochastic, i.e., $\sum_{j=1}^m{p_{ij}^{(k)}} = 1$, $\forall i$, and $\sum_{i=1}^m{p_{ij}^{(k)}} = 1$, $\forall j$. \label{assump:weights:doublystoch}
		
		\item (Symmetric weights) $p_{ij}^{(k)} = p_{ji}^{(k)}$, $\forall i,k$ and $p_{ii}^{(k)} = 1 - \sum_{j \neq i} {p_{ij}^{(k)}}$. \label{assump:weights:symmetric}
	\end{enumerate}
\end{assumption}

Considering the conditions mentioned in Assumption \ref{assump:weights}, and the definition of $p_{ij}^{(k)}$ in~\eqref{eqn:pDefinition}, a choice of parameters $\beta_{ij}^{(k)}$ that satisfy these assumptions are as follows:
\begin{equation} \label{eqn:betaDefinition}
	\beta_{ij}^{(k)} = \min{\left \lbrace \frac1{1 + d_i^{(k)}}, \frac1{1 + d_j^{(k)}} \right \rbrace},
\end{equation}
which is inspired by the Metropolis-Hastings algorithm \cite{boyd2004fastest}. Note that $p_{ij}^{(k)}$ also depends on $v_{ij}^{(k)}$, which was defined in~\eqref{eqn:indicatorSignal}.

\begin{assumption}[Convexity] \label{assump:convexity}

	\begin{enumerate}[label=(\alph*)]
		\item The local objective function at each device $i$, i.e.,  $F_i$, is convex:
		\begin{equation*}
		\hspace{-3mm}	F_i{\left( \mathbf{{w}'} \right)} \geq F_i{\left( \mathbf{w} \right)} + {\nabla F_i{\left( \mathbf{w} \right)}}^\top \left( \mathbf{{w}'} - \mathbf{w} \right),
		\end{equation*}
		$\forall \left( \mathbf{{w}'}, \mathbf{w} \right) \in \mathbb{R}^n \times \mathbb{R}^n$.
\\[-0.05in]
		\item The global objective function $F{\left( \mathbf{w} \right)} = \frac1m \sum_{i=1}^m{F_i{\left( \mathbf{w} \right)}}$ is convex, and thus has a non-empty minimizer set denoted by $\mathbf{W}^\star = \Argmin_{\mathbf{w} \in \mathbb{R}^n}{F{\left( \mathbf{w} \right)}}$, such that $F^\star = F{\left( \mathbf{w}^\star \right)}$ for any $\mathbf{w}^\star \in \mathbf{W}^\star$.
	\end{enumerate}
\end{assumption}

\begin{assumption}[Bounded gradients] \label{assump:boundedGradients}
	The gradient of each loss function $F_i$ is bounded, i.e., there exists a scalar $L_i > 0$ such that $\forall i\in\mathcal{M},~\mathbf{w} \in \mathbb{R}^n$,
    	\begin{equation*}
    		{\left\| \nabla F_i{ \left( \mathbf{w} \right)} \right\|}_2 \le L_i \le L,
    	\end{equation*}
    	where $L = \max_{i \in \mathcal{M}}{L_i}$. We define $L_\infty$ as the bound for the infinity norm of all $F_i$'s, i.e., ${\left\| \nabla F_i{\left( \mathbf{w} \right)} \right\|}_\infty \le L_\infty$, $\forall i$.
\end{assumption}

\begin{assumption}[Step sizes] \label{assump:stepsizes}
	All devices use the same step size for model training, which is diminishing over time and satisfies the following conditions:
		\begin{equation*}
			\lim_{k \to \infty}{\alpha^{(k)}} = 0, \quad \sum_{k=0}^\infty{\alpha^{(k)}} = \infty, \quad \sum_{k=0}^\infty{{\left( \alpha^{(k)} \right)}}^2 < \infty.
		\end{equation*}
\end{assumption}
In particular, setting $\alpha^{(k)} = \frac{a}{{\left( b+k \right)}^c}$ meets the criteria of the above assumption if $c \in (0.5, 1]$.

The previous assumptions are common in literature~\cite{wang2019adaptive,hosseinalipour2022multi}. In the next assumption, we introduce a relaxed version of graph connectivity requirements relative to existing work in distributed learning, which underscores the difference of our decentralized event-triggered FL method compared with traditional distributed optimization algorithms.

\begin{figure*}[t]
\vspace{-0mm} 
	\begin{subfigure}[b]{0.5\textwidth}
		\begin{center}
			\includegraphics[width=\textwidth]{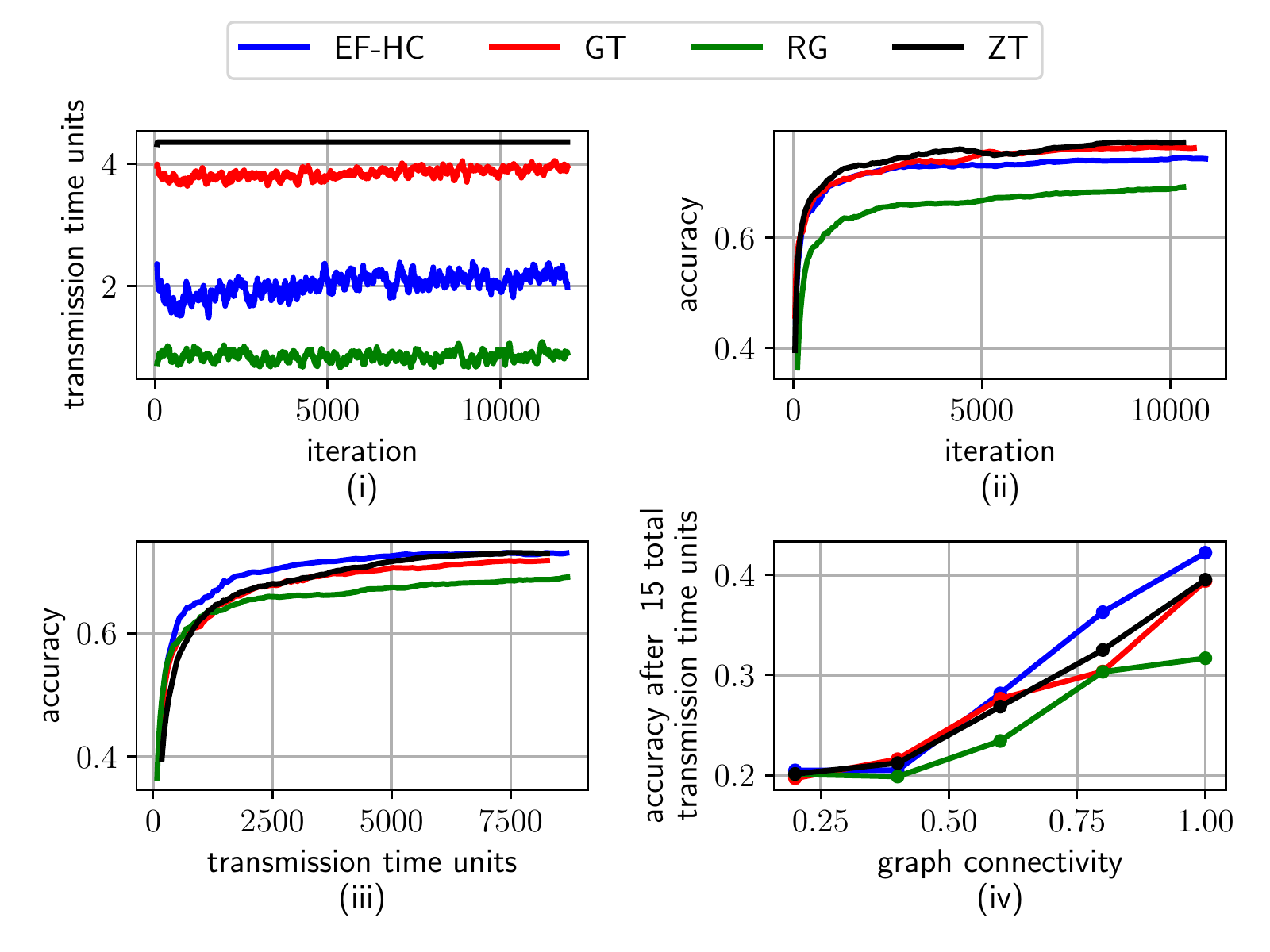}
		\end{center}
		\vspace{-5mm} 
		\caption{SVM classifier}
		\label{fig:sim:svm}
	\end{subfigure}
	\begin{subfigure}[b]{0.5\textwidth}
		\begin{center}
			\includegraphics[width=\textwidth]{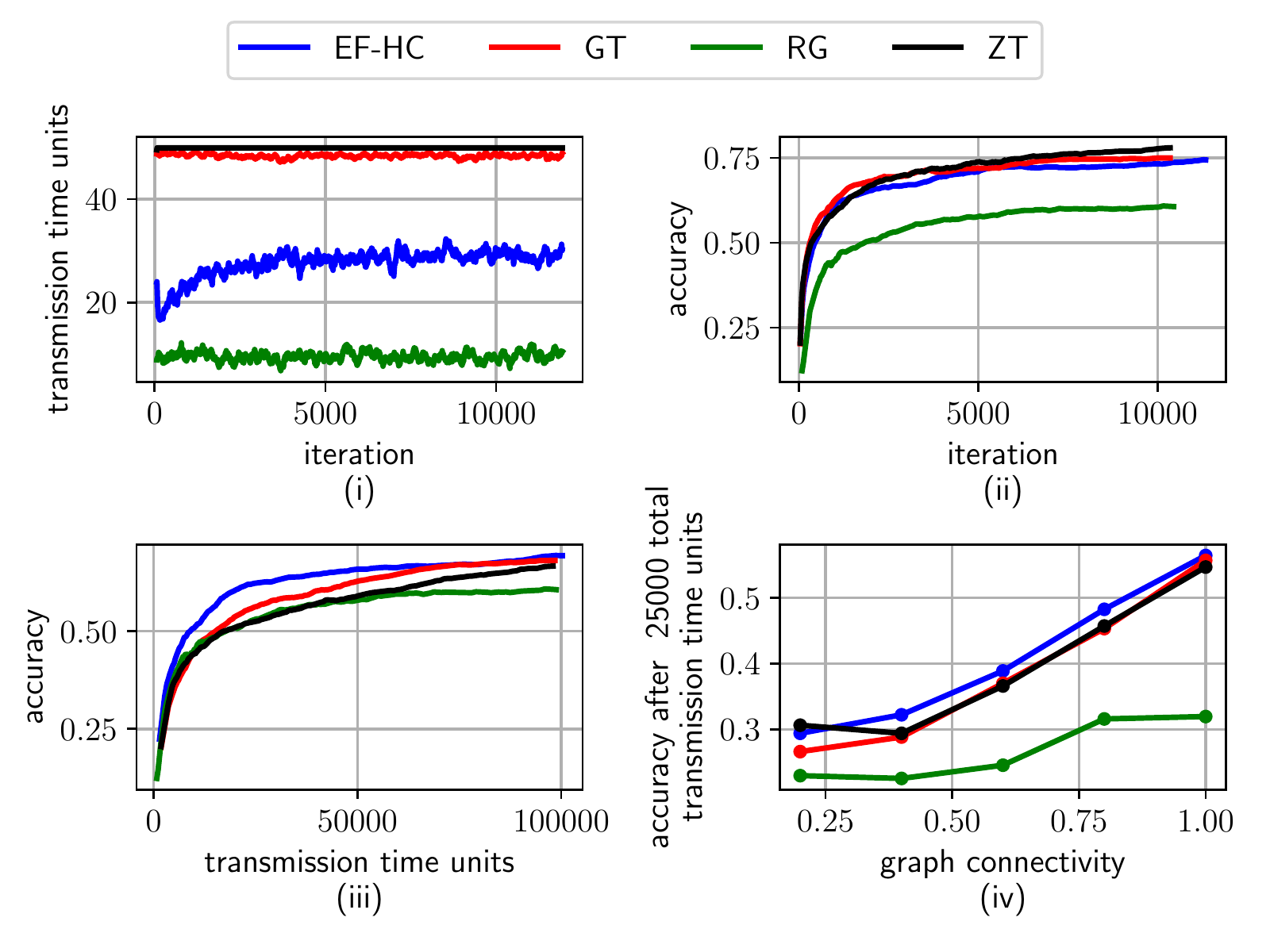}
		\end{center}
			\vspace{-5mm} 
		\caption{LeNet5 classifier}
		\label{fig:sim:lenet5}
	\end{subfigure}
	\vspace{-5mm}
	\caption{Performance comparison between our method ({\tt EF-HC}), global threshold (GT), zero threshold (ZT), and randomized gossip (RG) algorithms. The plots show (i) transmission time per iteration, (ii) accuracy per iteration, (iii) accuracy per transmission time, and (iv) accuracy after a certain number of transmissions with respect to graph connectivity.}
	\label{fig:sim}
		\vspace{-2mm}
\end{figure*}

\begin{assumption}[Network graph connectivity] \label{assump:connectivity}
	The physical network graph $\mathcal{G}^{(k)} = \left( \mathcal{M}, \mathcal{E}^{(k)} \right)$ satisfies the following:
	\begin{enumerate}[label=(\alph*)]
		\item There exists an integer $B_1 \geq 1$ such that the graph union of $\mathcal{G}^{(k)}$ from any arbitrary iteration $k$ to $k + B_1 - 1$, i.e., $\mathcal{G}^{( k : k+B_1-1 )} = {\left( \mathcal{M}, \cup_{s=0}^{B_1 - 1}{\mathcal{E}^{( k+s )}} \right)}$, is connected for any $k \geq 0$. \label{assump:conn:physicalconn}
		
		\item There exists an integer $B_2 \geq 1$ such that for every device $i$, triggering conditions for the broadcasting event occurs at least once every $B_2$ consecutive iterations $\forall k \geq 0$. This is equivalent to the following condition:
		\begin{equation*}
		\exists B_2\geq 1, \forall i:	\max{\lbrace v_i^{(k)},  v_i^{(k+1)}, \cdots, v_i^{( k+B_2-1 )} \rbrace} = 1.
		\end{equation*} \label{assump:conn:boundedintercom}
	\end{enumerate}
\end{assumption}

\vspace{-6mm}

Together, \ref{assump:conn:physicalconn} and \ref{assump:conn:boundedintercom} imply that each device $i$ broadcasts its information to its neighboring devices at least once every $B$ consecutive iterations, where $B = \left( l+2 \right) B_1$ in which $l B_1 < B_2 \le \left( l+1 \right) B_1$.\footnote{Note that in Algorithm \ref{alg:eventBased}, once two unconnected devices become connected, they exchange their parameters regardless of triggering conditions.}
Hence, the information flow graph $\mathcal{G}'^{(k)}$ is $B$-connected, i.e., $\mathcal{G}'^{( k : k+B-1 )} = {\left( \mathcal{M}, \cup_{s=0}^{B-1}{\mathcal{E}'^{( k+s )}} \right)}$ is connected, for any $k \geq 0$. It is important to note that we use $B$ only for convergence analysis, and it can have any arbitrarily large integer value. Therefore, we are not making strict connectivity assumptions on the underlying graph.

\vspace{-.1mm}
\section{Remarks on Hyperparameters}
\label{ssec:remark}
We make a few remarks on the hyperparameters used in Alg.~\ref{alg:eventBased}. Remark \ref{remark:normalization} elaborates on the choice of $q$ when calculating ${\left\| \mathbf{w}_i^{(k)} - \mathbf{\widehat{w}}_i^{(k)} \right\|}_q$, Remark~\ref{remark:rGuideline} discusses the choice of $r$ in the threshold $r \rho_i \gamma^{(k)}$, and Remark \ref{remark:ratesRelation} elaborates on the threshold decay rate $\gamma^{(k)}$ and the learning rate $\alpha^{(k)}$. More explanations of these remarks are deferred to Appendix~\ref{ssec:explain}.

\vspace{-2mm}
\begin{remark} \label{remark:normalization}
Factor ${\left( \frac1n \right)}^\frac1q$ in the event-triggering condition is a normalization factor, making the conditions independent of the model dimension $n$ and the norm $q \geq 1$ used.
\end{remark}

\vspace{-3mm}
\begin{remark} \label{remark:rGuideline}
The constant $r$ in the threshold of event-triggering condition is a hyperparameter to set the threshold $r \rho_i \gamma^{(k)}$ to a value comparable with ${\left( \frac1n \right)}^\frac1q {\left\| \mathbf{w}_i^{(k)} - \mathbf{\widehat{w}}_i^{(k)} \right\|}_q$. The value of this constant can be chosen as follows:
\vspace{-1mm}
	\begin{equation} \label{eqn:rGuideline}
		\begin{aligned}
			r  = \frac{\alpha^{(0)}}{\gamma^{(0)}} \frac1\rho K L_\infty
			 = \frac1\rho K L_\infty \quad \text{if} \quad \alpha^{(0)} = \gamma^{(0)},
		\end{aligned}
		\vspace{-1mm}
	\end{equation}
in which $K$ is an approximation on the number of local iterations between aggregation events, $\frac1\rho$ is an approximation of $\frac1m \sum_{i\in\mathcal{M}}{\frac1{\rho_i}}$, and $L_\infty$ is an upper bound obtained via the relation ${\left\| \nabla F_i{ \left( x \right)} \right\|}_\infty \le L_\infty$ for all $i$ from Assumption \ref{assump:boundedGradients}.
\end{remark}

\begin{remark} \label{remark:ratesRelation}
To ensure sporadic aggregations at each device in {\tt EF-HC}, the learning rate $\alpha^{(k)}$ and threshold decay rate $\gamma^{(k)}$ should satisfy the following conditions:
	\begin{enumerate}[label=(\alph*)]
		\item $\lim_{k \to \infty}{\frac{\gamma^{(k)}}{\alpha^{(k)}}} = \Omega$, where $ \Omega $ is a finite positive constant. \label{remark:ratesRelation:constant}
		
		\item $\lim_{k \to \infty}{\frac{\nicefrac{\gamma^{(k)}}{\gamma^{(0)}}}{\nicefrac{\alpha^{(k)}}{\alpha^{(0)}}}} = 1$, i.e., $\Omega = \frac{\gamma^{(0)}}{\alpha^{(0)}}$. \label{remark:ratesRelation:1}
	\end{enumerate}
\end{remark}

The conditions in Remark \ref{remark:ratesRelation} ensure that aggregation events (Event 3) neither cease completely nor occur continuously after a while, but instead are executed only when our proposed triggering condition is met (see  Appendix~\ref{ssec:explain} for details). To satisfy these conditions, first the learning rate $\alpha^{(k)}$ should be chosen to meet the criteria of Assumption \ref{assump:stepsizes}, and then $\gamma^{(k)}$ should be chosen to satisfy the above conditions. One choice that satisfies these conditions is $\gamma^{(k)} = \alpha^{(k)}$.

\section{Numerical Results}
\label{sec:simulation}
We now conduct numerical experiments to validate our methodology. We explain our simulation setup in Sec.~\ref{setup} and provide the results and discussion in Sec.~\ref{discussion}.

\subsection{Simulation Setup}\label{setup}
We evaluate our proposed methodology classification tasks on the Fashion-MNIST image recognition dataset \cite{xiao2017fashion}. We employ support vector machine (SVM) and the LeNet5 neural network model as classifiers; SVM satisfies Assumption~\ref{assump:convexity} while LeNet5 (and deep learning models in general) does not.

We consider a network of $m=10$ devices, where the topology
is generated according to a random geometric graph with connectivity $0.4$~\cite{hosseinalipour2022multi}. To generate non-i.i.d. data distributions across devices, each device only contains samples of Fashion-MNIST from a fraction of the $10$ labels. For SVM and LeNet5, we consider 1 and 2 labels/device, respectively.

We set the average link bandwidth to $5000$.
We introduce a resource heterogeneity measure $H$, $0 \le H < 1$, which we use to generate networks with two types of devices: (i) ``weak," which have outgoing links with an average bandwidth of $1000$, and (ii) ``powerful," which have an average outgoing link bandwidth of $\frac{5000 - 1000H}{1-H}$. We set $H=0.8$ for LeNet5 and $H=0.4$ for SVM.

In each experiment, the learning rate is selected as $\alpha^{(k)} = \frac1{\sqrt{1+k}}$, and threshold decay rate is set to $\gamma^{(k)} = \alpha^{(k)}$, satisfying the conditions in Remark \ref{remark:ratesRelation}. The $2$-norm is used for the event-triggering conditions (see Remark \ref{remark:normalization}), and $r=5000 \times 10^{-2}$ following the guidelines of Remark \ref{remark:rGuideline}.

At iteration $k$, we define a resource utilization score as $\frac1m \sum_{i=1}^m{\frac{\sum_{j=1}^m{v_{ij}^{(k)}}}{d_i^{(k)}} \rho_i n}$. The term $\frac{\sum_{j=1}^m{v_{ij}^{(k)}}}{d_i^{(k)}}$ is the outgoing link utilization, and therefore this score is a weighted average of link utilization, penalizing devices with larger $\rho_i$. For our proposed method where $\rho_i = \frac1{b_i}$, this score is the same as the average transmission time, i.e., $\frac1m \sum_{i=1}^m{\frac{\sum_{j=1}^m{v_{ij}^{(k)}}}{d_i^{(k)}} \frac{n}{b_i}}$.

\subsection{Results and Discussion}\label{discussion}
We compare the performance of our method {\tt EF-HC} against three baseline methods: (i) distributed learning with aggregations at every iteration, i.e., zero thresholds (denoted by \textit{ZT}), (ii) decentralized event-triggered FL with the same global threshold across all devices (denoted by \textit{GT}), and (iii) randomized gossip algorithm where each device engages in communication with probability of $\frac1m$ at each iteration~\cite{pu2021distributed} (denoted by \textit{RG}). The performance of our method against these baselines for each classifier is depicted in Fig. \ref{fig:sim}.

Figs. \ref{fig:sim:svm}-(i) and \ref{fig:sim:lenet5}-(i) show the average transmission time units each algorithm requires per training iteration. As can be observed, {\tt EF-HC} results in less transmission delay compared to \textit{ZT} and \textit{GT}, 
which helps to resolve the impact of stragglers by not requiring the same amount of communications and aggregations from devices with less available bandwidth.
Note that although less transmission delay per iteration is desirable for a decentralized FL algorithm, this runs the risk of degrading the performance of the classification task in terms of accuracy when the data distribution across devices is non-i.i.d. Thus, a good comparison between multiple decentralized algorithms is to consider the accuracy reached per transmission time units. In this regard, although \textit{RG} achieves less transmission delay per iteration compared to our method, Figs. \ref{fig:sim:svm}-(iii) and \ref{fig:sim:lenet5}-(iii) reveal that it achieves substantially lower model performance, indicating that our method strikes an effective balance between these objectives.

The average accuracy of devices per iteration is plotted in Figs. \ref{fig:sim:svm}-(ii) and \ref{fig:sim:lenet5}-(ii). These plots are indicative of processing efficiency since they evaluate the accuracy of algorithms per number of gradient descent computations. As expected, the baseline algorithm \textit{ZT} is able to achieve the highest accuracy per iteration since it does not take resource efficiency into account, and thus sacrifices network resources to reach a better accuracy. In these plots, we show that unlike \textit{RG}, the performance of our proposed method {\tt EF-HC} as well as \textit{GT} do not considerably degrade although they use less communication resources, as will be discussed next.

Figs. \ref{fig:sim:svm}-(iii) and \ref{fig:sim:lenet5}-(iii) are perhaps the most critical results, as they assess the accuracy vs. communication time tradeoff. We see that our algorithm {\tt EF-HC} can achieve a higher accuracy while using less transmission time compared to all the baselines, both for the SVM classifier and LeNet5, i.e., with and without the model convexity assumption from our convergence analysis. These plots reveal that our method can adapt to non-i.i.d data distributions across the devices, which is an important characteristic for FL algorithms~\cite{kairouz2021advances}, and achieve a better accuracy as compared to the baselines given a fixed transmission time, i.e., under a fixed network resource consumption.

Finally, we evaluate the effect of network connectivity on our method and baseline methods in Figs. \ref{fig:sim:svm}-(iv) and \ref{fig:sim:lenet5}-(iv) \footnote{For the LeNet5 classifier, we change the simulation setup and set $r = 5000\times10^{-3}$, and let the devices to have samples from only $1$ labels/device.}. Since the graphs are generated randomly in our simulations, we have taken the average performance of all four algorithms over $5$ Monte Carlo instances to reduce the effect of random initialization on the results. It can be observed that higher network connectivity improves the convergence speed of our method and most of the baselines, as expected. Importantly, however, we see that our method has the highest improvement per increase in connectivity.

\section{Conclusion and Future Work}
In this paper, we developed a novel methodology for event-triggered FL with heterogeneous communication thresholds ({\tt EF-HC}). {\tt EF-HC} introduces a scenario where the conventional centralized model aggregations in FL are carried out in a decentralized manner via P2P communications among the devices. To further alleviate the burden of a centralized scheduler and take into account resources heterogeneity across the devices, it considers event-triggered communications with heterogeneous communication thresholds. We conducted a theoretical analysis of {\tt EF-HC} and demonstrated that model training under {\tt EF-HC} asymptotically achieves the global optimal model for standard assumptions in distributed learning. Future work can focus on deriving optimal/data-driven algorithms for setting the event-triggering communication conditions under different network settings.

\bibliographystyle{ieeetr}
\bibliography{root}

\vspace{-1mm}
\appendix

\vspace{-2mm}
\subsection{Proof of Theorem 1}
\label{ssec:proof}
Rewriting the event-based updates of Algorithm \ref{alg:eventBased}, we get
\vspace{-1mm}
\begin{equation} \label{eqn:iterBased}
	\begin{gathered}
		\begin{aligned}
			\mathbf{w}_i^{(k+1)} = \mathbf{w}_i^{(k)} & + \sum_{j \in \mathcal{N}'^{(k)}_i}{\beta_{ij}^{(k)} \left( \mathbf{w}_j^{(k)} - \mathbf{w}_i^{(k)} \right) v_{ij}^{(k)}}
			\\
			& - \alpha^{(k)} \mathbf{g}_i{\left( \mathbf{w}_i^{(k)} \right)},
		\end{aligned}
		\\
		\mathbf{\widehat{w}}_i^{(k+1)} = \mathbf{\widehat{w}}_i^{(k)} \left( 1 - v_i^{(k)} \right) + \mathbf{w}_i^{(k)} v_i^{(k)}.
	\end{gathered}
	\vspace{-1mm}
\end{equation}

Rearranging the relations in \eqref{eqn:iterBased}, we have
\begin{equation} \label{eqn:recursiveRelation}
	\begin{aligned}
		& \begin{aligned}
			\mathbf{w}_i^{(k+1)} = & \left( 1 - \sum_{j=1}^m{\beta_{ij}^{(k)} v_{ij}^{(k)}} \right) \mathbf{w}_i^{(k)}
			\\
			& + \sum_{j=1}^m{\beta_{ij}^{(k)} v_{ij}^{(k)} \mathbf{w}_j^{(k)}} - \alpha^{(k)} \mathbf{g}_i{\left( \mathbf{w}_i^{(k)} \right)}
		\end{aligned}
		\\
		& = \sum_{j=1}^m{p_{ij}^{(k)} \mathbf{w}_j^{(k)}} - \alpha^{(k)} \mathbf{g}_i{\left( \mathbf{w}_i^{(k)} \right)}.
	\end{aligned}
\end{equation}

Next, we collect the vectors of all devices that were previously introduced into matrix form as follows:
{\small
$
		\mathbf{W}^{(k)} =
		\begin{bmatrix}
			\mathbf{w}_1^{(k)}
			& \cdots &
			\mathbf{w}_m^{(k)}
		\end{bmatrix}^\top$, $
		\mathbf{\widehat{W}}^{(k)} =
		\begin{bmatrix}
			\mathbf{\widehat{w}}_1^{(k)}
			& \cdots &
			\mathbf{\widehat{w}}_m^{(k)}
		\end{bmatrix}^\top$, 
		$
		\mathbf{G}^{(k)} =
		\begin{bmatrix}
			\mathbf{g}_1{\left( \mathbf{w}_1^{(k)} \right)}
			& \cdots &
			\mathbf{g}_m{\left( \mathbf{w}_m^{(k)}  \right)}
		\end{bmatrix}^\top$,}
		$
		\mathbf{P}^{(k)} = [p_{ij}^{(k)}]_{1\leq i,j\leq m}$.

Now, we transform the recursive update rules of \eqref{eqn:recursiveRelation} into matrix form to get the following relationship:
\begin{equation} \label{eqn:recursiveRelationMatrix}
	\mathbf{W}^{(k+1)} = \mathbf{P}^{(k)} \mathbf{W}^{(k)} - \alpha^{(k)} \mathbf{G}^{(k)}.
\end{equation}

The recursive expression in \eqref{eqn:recursiveRelationMatrix} has been investigated before \cite{nedic2009distributed}. In the following, we build upon some lemmas from prior work given our assumptions in Sec. \ref{sec:convergence:sub:assumptions} to obtain the final result of the theorem.

Starting from iteration $s$, where $s \le k$, we have
\vspace{-1mm}
\begin{equation} \label{eqn:explicitRelationMatrix}
	\begin{gathered}
	    \begin{aligned}
    	    \mathbf{W}^{(k+1)} = \mathbf{P}^{( k:s )} \mathbf{W}^{(s)} & - \sum_{r=s+1}^k{\alpha^{(r-1)} \mathbf{P}^{( k:r )} \mathbf{G}^{(r-1)}}
    	    \\
    	    & - \alpha^{(k)} \mathbf{G}^{(k)},
	    \end{aligned}
	    \\
	    \mathbf{P}^{( k:s )} = \mathbf{P}^{(k)} \mathbf{P}^{(k-1)} \cdots \mathbf{P}^{(s+1)} \mathbf{P}^{(s)}.
	\end{gathered}
	\vspace{-1mm}
\end{equation}

If we let $s = 0$ in \eqref{eqn:explicitRelationMatrix}, we get an explicit relationship for the model parameters at iteration $k$ with respect to their initial values. Focusing on the parameters of each device $i$ (row $i$ of $\mathbf{W}^{(k+1)}$), we get
\vspace{-1mm}
\begin{equation} \label{eqn:wFromZero}
  \hspace{-2mm}
    \begin{aligned}
	    &\mathbf{w}_i^{(k+1)} =  \sum_{j=1}^m{p_{ij}^{( k:0 )} \mathbf{w}_j^{(0)}}
	    \\
	    & \;\;\; - \sum_{r=1}^k{\alpha^{(r-1)} \sum_{j=1}^m{p_{ij}^{( k:r )} \mathbf{g}_j{\left( \mathbf{w}_j^{(r-1)} \right)}}} - \alpha^{(k)} \mathbf{g}_i{\left( \mathbf{w}_i^{(k)} \right)}.
    \end{aligned}
    \hspace{-2mm}
    \vspace{-1mm}
\end{equation}

To analyze the local model consensus, we define the average model $\mathbf{\bar{w}}^{(k)}$ as
\vspace{-1mm}
\begin{equation*}
    \mathbf{\bar{w}}^{(k)} = \frac1m \sum_{i=1}^m{\mathbf{w}_i^{(k)}} = \frac1m \mathbf{1}_m^\top \mathbf{W}^{(k)}.
    \vspace{-1mm}
\end{equation*}

The recursive relation for $\mathbf{\bar{w}^{(k)}}$ using \eqref{eqn:recursiveRelationMatrix} and the stochasticity of $\mathbf{P}^{(k)}$ is
\vspace{-1mm}
\begin{equation*}
    \mathbf{\bar{w}}^{(k+1)} = \frac1m \mathbf{1}_m^\top \mathbf{W}^{(k+1)} = \mathbf{\bar{w}}^{(k)} - \frac{\alpha^{(k)}}m \mathbf{1}_m^\top \mathbf{G}^{(k)}.
    \vspace{-1mm}
\end{equation*}

Also, the explicit relationship connecting $\mathbf{\bar{w}}^{(k+1)}$ to its corresponding value at iteration $0$ can be calculated using \eqref{eqn:wFromZero} together with the stochasticity of $\mathbf{P}^{( k:0 )}$:
\vspace{-1mm}
\begin{equation} \label{eqn:wbarFromZero}
    \mathbf{\bar{w}}^{(k+1)} = \mathbf{\bar{w}}^{(0)} - \frac1m \sum_{r=1}^{k+1}{\alpha^{(r-1)} \sum_{j=1}^m{\mathbf{g}_j{\left( \mathbf{w}_j^{(r-1)} \right)}}}.
    \vspace{-1mm}
\end{equation}

Part \ref{lemma:consensus:0} of the following lemma shows that model parameters $\mathbf{w}_i^{(k)}$ of each device $i$ asymptotically converge to $\mathbf{\bar{w}}^{(k)}$, thus reaching consensus as $k \to \infty$.

\begin{lemma} [Follows from Lemma 8 of \cite{nedic2010constrained}] \label{lemma:consensus}
	Let the sequence $\lbrace \mathbf{w}_i^{(k)} \rbrace$ be generated by iteration \eqref{eqn:wFromZero} and the sequence $\left \lbrace \mathbf{\bar{w}}^{(k)} \right \rbrace$ be generated by \eqref{eqn:wbarFromZero}. Then $\forall i \in \mathcal{M}$ we have
	\begin{enumerate}[label=(\alph*)]
	    \item $\lim_{k \to \infty}{{\left\| \mathbf{w}_i^{(k)} - \mathbf{\bar{w}}^{(k)} \right\|}_2} = 0$, if the step size satisfies $\lim_{k \to \infty}{\alpha^{(k)}} = 0$, and \label{lemma:consensus:0}
	    
	    \item $\sum_{k=1}^\infty{\alpha^{(k)} {\left\| \mathbf{w}_i^{(k)} - \mathbf{\bar{w}}^{(k)} \right\|}_2} < \infty$, if the step size satisfies $\sum_{k=0}^\infty{{\left( \alpha^{(k)} \right)}^2} < \infty$. \label{lemma:consensus:const}
	\end{enumerate}
\end{lemma}

We next move on to show that $\mathbf{\bar{w}}^{(k)}$ under our method asymptotically converges to the optimizer of the global loss. First, we provide the following lemma, which reveals the relationship between  $F{\left( . \right)}$ evaluated at $\mathbf{\bar{w}}^{(k)}$ and $\mathbf{w}_i^{(k)}$.

\begin{lemma}[Follows from Lemma 6 in \cite{nedic2010constrained}] \label{lemma:objective_model_relation}
    Let the sequence $\lbrace \mathbf{w}_i^{(k)} \rbrace$ be generated by iteration \eqref{eqn:wFromZero} $\forall i \in \mathcal{M}$ and the sequence $\lbrace \mathbf{\bar{w}}^{(k)} \rbrace$ be generated by iteration \eqref{eqn:wbarFromZero}. If Assumptions \ref{assump:convexity} and \ref{assump:boundedGradients} hold, we have
    \begin{equation*}
        \begin{aligned}
            \frac{2 \alpha^{(k)}}m & \left( F{\left( \mathbf{\bar{w}}^{(k)} \right)} - F{\left( \mathbf{w}_i^{(k)} \right)} \right)
            \\
            & \le {\left\| \mathbf{\bar{w}}^{(k)} - \mathbf{w}^{(k)} \right\|}_2^2 - {\left\| \mathbf{\bar{w}}^{(k+1)} - \mathbf{w}_i^{(k)} \right\|}_2^2
            \\
            & + \frac{L^2}m {\left( \alpha^{(k)} \right)}^2 + \frac{4L}m \alpha^{(k)} \sum_{j=1}^m{{\left\| \mathbf{\bar{w}}^{(k)} - \mathbf{w}_j^{(k)} \right\|}_2}.
        \end{aligned}
    \end{equation*}
\end{lemma}
Finally, we only need to show that the average of models $\mathbf{\bar{w}}^{(k)}$ asymptotically optimizes the global loss. To prove this, we take the summation of the relation in Lemma \ref{lemma:objective_model_relation} from $k=0$ to $\infty$, and then use the results of Lemma \ref{lemma:consensus}-\ref{lemma:consensus:const} alongside the step size conditions $\lim_{k \to \infty}{\alpha^{(k)}} = 0$ and $\sum_{k=0}^\infty{{\left( \alpha^{(k)} \right)}^2} < \infty$. It follows that $\lim_{k \to \infty}{F{\left( \mathbf{\bar{w}}^{(k)} \right)} - F^\star} = 0$.

\subsection{Further Explanation of Remarks 1-3}
\label{ssec:explain}

\textbf{Remark \ref{remark:normalization}.}
For the $q$-norm of a vector $\mathbf{w} \in \mathbb{R}^n$, we have
\begin{equation} \label{eqn:normBounds}
{\left\| \mathbf{w} \right\|}_u \le {\left\| \mathbf{w} \right\|}_q \le n^{\frac1q - \frac1u} {\left\| \mathbf{w} \right\|}_u,
\end{equation}
where $1 \le q < u$. Also note that based on the way Algorithm \ref{alg:eventBased} defines the event-triggering conditions, the relation $C {\| \mathbf{w}_i^{\left( k \right)} - \mathbf{\widehat{w}}_i^{\left( k \right)} \|}_q < r \rho_i \gamma^{\left( k \right)}$ holds at every iteration, since otherwise an event will be triggered to ensure it. $C$ is a normalization factor to be derived here.

\begin{enumerate}[label=(\alph*)]
	\item Considering all the norms that can be used for a vector, it is only the $\infty$-norm that does not depend on the dimension of the vector. Thus, a model-invariant event-triggering condition would result in the relation ${\| \mathbf{w}_i^{\left( k \right)} - \mathbf{\widehat{w}}_i^{\left( k \right)} \|}_\infty < r \rho_i \gamma^{\left( k \right)}$ holding with $C=1$. \\[-0.05in]
	
	\item To not be constrained by the $\infty$-norm over the choice of $q$ in $C {\| \mathbf{w}_i^{\left( k \right)} - \mathbf{\widehat{w}}_i^{\left( k \right)} \|}_q < r \rho_i \gamma^{\left( k \right)}$, and to still ensure invariance over the model dimension $n$, we can write the following by letting $u \to \infty$ in \eqref{eqn:normBounds}:
	\begin{equation*}
	{\left( \frac1n \right)}^\frac1q {\left\| \mathbf{w}_i^{\left( k \right)} - \mathbf{\widehat{w}}_i^{\left( k \right)} \right\|}_q \le {\left\| \mathbf{w}_i^{\left( k \right)} - \mathbf{\widehat{w}}_i \right\|}_\infty < r \rho_i \gamma^{\left( k \right)}.
	\end{equation*}
\end{enumerate}

\textbf{Remark \ref{remark:rGuideline}. }
Based on \eqref{eqn:explicitRelationMatrix}, we can obtain the following relationship between the state of device $i$ from iteration $s$ to $k \geq s$:
\begin{equation*}
	\begin{aligned}
    	\mathbf{w}_i^{\left( k \right)} = & \sum_{j=1}^m{p_{ij}^{\left( k-1:s \right)} \mathbf{w}_j^{\left( s \right)}} - \alpha^{\left( k-1 \right)} \mathbf{g}_i{\left( \mathbf{w}_i^{(k-1)} \right)}
    	\\
    	& - \sum_{r=s+1}^{k-1}{\alpha^{\left( r-1 \right)} \sum_{j=1}^m{p_{ij}^{\left( k-1:r \right)} \mathbf{g}_j{\left( \mathbf{w}_j^{(r-1)} \right)}}}.
	\end{aligned}
\end{equation*}

Assuming no aggregation events occur at device $i$ or its neighbors from iteration $s$ to iteration $k$, we will have: (i) $\mathbf{\widehat{w}}_i^{\left( k \right)} = \mathbf{w}_i^{\left( s \right)}$; and (ii) $p_{ii}^{\left( k-1:r \right)} = 1$ and $p_{ij}^{\left( k-1:r \right)} = 0$ for all $j \neq i$ and $s \le r \le k-1$. As a result, we get
\begin{equation*}
	\mathbf{w}_i^{\left( k \right)} = \mathbf{w}_i^{\left( s \right)} - \sum_{r=s}^{k-1}{\alpha^{\left( r \right)} \mathbf{g}_i{\left( \mathbf{w}_i^{(r)} \right)}}.
\end{equation*}

In other words, device $i$ solely conducts SGD from iteration $s$ to $k$, and thus
\begin{equation} \label{eqn:ratesRelation}
	\begin{aligned}
    	{\left\| \mathbf{w}_i^{\left( k\right)} - \mathbf{\widehat{w}}_i^{\left( k\right)} \right\|}_q & \le \sum_{r=s}^{k-1}{\alpha^{\left( r \right)} n^{\frac1q} {\left\| \mathbf{g}_i{\left( \mathbf{w}_i^{(r)} \right)} \right\|}_\infty}
    	\\
    	& \le \alpha^{\left( s \right)} n^{\frac1q} \left( k-s \right) L_\infty.
	\end{aligned}
\end{equation}

The expression above gives us a guideline on selecting the hyperparameter $r$. Since $r$ has a constant value throughout the training process, we select it in a way to have our desired behavior from the early iterations, i.e., $s=0$. Considering the extreme case where maximum steps are taken to update $\mathbf{w}_i^{\left( k \right)}$, i.e., steps of size $\alpha^{(0)} {\| \mathbf{g}_i{( \mathbf{w}_i^{(0)} )} \|}_q \approx \alpha^{(0)} n^{\frac1q} L_\infty$, we set $r$ to a value such that it would take approximately $K$ iterations with maximum steps before the event-triggering condition is reached. Note that for the threshold decay rate $\gamma^{\left( k \right)}$, its extreme case value $\gamma^{(0)}$ is considered as well:
\begin{equation} \label{eqn:ratesRelation2}
	\begin{gathered}
    	{\left( \frac1n \right)}^{\frac1q} \alpha^{(0)} n^{\frac1q} K L_\infty = r \rho_i \gamma^{(0)},
    	\\
    	r = \frac{\alpha^{(0)}}{\gamma^{(0)}} \frac1{\rho_i} K L_\infty.
	\end{gathered}
\end{equation}

However, $r$ should be a global variable that has the same value across all devices. Thus, we take the average of the relation above across all devices
\begin{equation*}
	r = \frac{\alpha^{(0)}}{\gamma^{(0)}} \left( \frac1m \sum_{i=1}^m{\frac1{\rho_i}} \right) K L_\infty.
\end{equation*}

Since in our fully-decentralized setting there is no central server with the knowledge of each $\rho_i$, calculating $\frac1m \sum_{i=1}^m{\frac1{\rho_i}}$ exactly is not possible. Thus, we replace that term with an estimate $\frac1\rho$ to get~\eqref{eqn:rGuideline}.

\textbf{Remark \ref{remark:ratesRelation}. } Expression \eqref{eqn:rGuideline} in Remark \ref{remark:rGuideline} was used to derive a value for the constant $r$. We use similar arguments to find a relationship between the learning rate $\alpha^{(k)}$ and the threshold decay rate $\gamma^{(k)}$. Using \eqref{eqn:ratesRelation} and \eqref{eqn:ratesRelation2} and solving for ${\Delta k}_i^u = k_i^{u+1} - k_i^u$, where $k^u_i$ denotes the iteration where the $\{u\}$-th aggregation event occurs at device $i$, gives us
\begin{equation*}
	{\Delta k}_i^u = \frac{r \rho_i}{L_\infty} \frac{\gamma^{(k)}}{\alpha^{(k)}}.
\end{equation*}

We are interested in the asymptotic behavior of ${\Delta k}_i^u$ as $k \to \infty$. $\lim_{k \to \infty}{{\Delta k}_i^u} = \infty$ implies that aggregation events become less frequent as time goes by and stop after a while. This contradicts Assumption \ref{assump:connectivity}-\ref{assump:conn:boundedintercom} (bounded intercommunication intervals) and hence should be avoided. There is no particular issue when $\lim_{k \to \infty}{{\Delta k}_i^u} = 0$ in terms of consensus, as it implies an aggregation occurs at every iteration after a while. However, we avoid this situation as it defeats our purpose of sporadic event-triggered communications. Therefore, we aim for having a finite constant value for $\lim_{k \to \infty}{{\Delta k}_i^u}$ (this constant is equal to $K$, which is the approximate number of iterations between aggregation events), and thus
\begin{equation*}
	K = \frac{r \rho_i}{L_\infty} \lim_{k \to \infty}{\frac{\gamma^{(k)}}{\alpha^{(k)}}} \Rightarrow \lim_{k \to \infty}{\frac{\gamma^{(k)}}{\alpha^{(k)}}} = \frac{K L_\infty}{r \rho_i}.
\end{equation*}
So, the decay rate of $\gamma^{(k)}$ and $\alpha^{(k)}$ should be the same.
We next substitute the value of $r$ derived in \eqref{eqn:rGuideline} to get
\begin{equation*}
	\lim_{k \to \infty}{\frac{\nicefrac{\gamma^{(k)}}{\gamma^{(0)}}}{\nicefrac{\alpha^{(k)}}{\alpha^{(0)}}}} = \frac{\rho}{\rho_i}.
\end{equation*}
Similar to the argument made in Remark \ref{remark:rGuideline}, since both $\gamma^{(k)}$ and $\alpha^{(k)}$ are global variables, we take the average of the relationship above to obtain
\begin{equation*}
	\lim_{k \to \infty}{\frac{\nicefrac{\gamma^{(k)}}{\gamma^{(0)}}}{\nicefrac{\alpha^{(k)}}{\alpha^{(0)}}}} = \rho \left( \frac1m \sum_{i=1}^m{\frac1{\rho_i}} \right) = 1.
\end{equation*}

\addtolength{\textheight}{-3cm}

\end{document}